\documentclass{llncs}
\usepackage{graphicx}      % include this line if your document contains figures
\usepackage{amsmath} 
\usepackage{amssymb}
\usepackage{tensor} 
\usepackage{cite}
\usepackage{graphicx}
\usepackage{floatrow}
\newfloatcommand{capbtabbox}{table}[][\FBwidth]

\usepackage{blindtext}

\newcommand{\Proj}{\mathrm{P}{\,}}

\newcommand{\ProjNonOrth}[2]{\tensor*{\Proj}{^{#1}_{#2}}}
\newcommand{\LaplaceBeltrami}{\mathrm{\Delta_{{LB}}}}
\newcommand{\partderiv}[2]{\partial_{#2} {#1}}

\newcommand{\toreal}{\rightarrow\bbbr}

\newcommand{\CovariantDiffManif}[1]{\nabla^{#1}}
\newcommand{\CovariantDerivManif}[2]{\tensor*{\nabla}{^{#1}_{#2}}}
\newcommand{\CovariantDiff}{\nabla}
\newcommand{\CovariantDeriv}[1]{\nabla_{#1}}
\newcommand{\Diff}{\mathrm{d}}
\newcommand{\TangentSpaceArg}[2]{{T_{#2}}{#1}}

\newcommand{\TangentBundle}[1]{T{#1}}
\newcommand{\CotangentBundle}[1]{\tensor*{T}{^{*}}{#1}}
\newcommand{\FRScalar}{BR_{\mathrm{scalar}}}
\newcommand{\FRMean}{BR_{\mathrm{mean}}}
\newcommand {\tr}{{\,}\mathrm{tr}{\,}}
\newcommand {\Preimage}[2]{{#2}^*{#1}}

\newcommand \TPreimage[2]{\Preimage{\TangentBundle{#1}}{#2}}

\newcommand {\bigeps}{\mathcal{E}}

\begin{document}

\title{A Riemanian Approach to Blob Detection in Manifold-Valued Images}%

\author{Aleksei Shestov\inst{1} \and Mikhail Kumskov\inst{1}} 
 \institute{Lomonosov Moscow State University, Faculty of Mechanics and Mathematics, Russia, 119991, Moscow, GSP-1, 1 Leninskiye Gory, Main Building
\\
\email{\{shestov.msu,mkumskov\}@gmail.com}}

\maketitle              % typeset the title of the contribution

\begin{abstract}
This paper is devoted to the problem of blob detection in manifold-valued images. Our solution is based on new definitions of blob response functions. We define the blob response functions by means of curvatures of an image graph, considered as a submanifold. We call the proposed framework Riemannian blob detection. We prove that our approach can be viewed as a generalization of the grayscale blob detection technique. An expression of the Riemannian blob response functions through the image Hessian is derived. We provide experiments for the case of vector-valued images on 2D surfaces: the proposed framework is tested on the task of chemical compounds classification.
 \end{abstract}  
\section{Introduction}
Blob detection \cite{blob} is a widely used method of keypoints detection in grayscale images. Informally speaking, blob detection aims to find ellipse-like regions of different sizes with “similar” intensity inside. Blobs are sought as local extremums of a blob response function. Several color blob detection algorithms were proposed in \cite{ColorBlob,GROM}. Blob detection has applications in 3D face recognition, object recognition, panorama stitching, 3D scene modeling, tracking, action recognition, medical images processing, etc.
\\
Our goal is to propose a blob detection framework for the general setting of an image being a map between Riemannian manifolds. Our approach is based on a definition of blob response functions by means of image graph curvatures. Furthermore, we derive the expression of Riemannian blob response functions through image Hessian. This expression shows that Riemannian blob detection coincides with the classical blob detection framework for the grayscale case. Also this expression provides a more convenient way to calculate Riemannian blob response functions for vector- and manifold-valued images.
\\
Research of connections between image processing methods and image graph geometry is of its own interest. This research helps deeply understand traditional methods, provides insights and gives natural generalizations of classical methods to vector-valued and manifold-valued images \cite{Saucan,Kimmel,Batard}. Connections between the blob response functions and image graph curvatures were mentioned in papers \cite{BlobCurv1,BlobCurv2}. Our work is the first to accurately analyze this question in the general setting.

\subsubsection{Contributions:}
\begin{enumerate}
\item We are the first to provide a blob detection framework for the general setting of an image being a map between manifolds. This framework can be viewed as a generalization of grayscale blob detection. Our framework provides blob response functions for the previously uncovered problems: blob detection in color images on manifold domain and blob detection in manifold-valued images (both on Euclidian and manifold domains). 
\item We are the first to analyze connections between the blob response functions and curvatures of image graph both for Euclidian and manifold domains.
\item The experiments on the task of chemical compounds classification show the effectiveness of our approach for the case of vector-valued images on 2d surfaces.  
\end{enumerate}

\section{The Problem Introduction}
Blob detection was firstly proposed for grayscale images on 2D Euclidian domain \cite{blob}. In \cite{ScalarBlob3D} blob detection was generalized to 2D surfaces. Several approaches to generalization of blob detection to color case were proposed in \cite{ColorBlob,GROM}. However, these approaches are based on global or local conversion of a color image to the grayscale, so they can't be used for manifold-valued images.
\\
Consider a grayscale image $I(x):X \toreal$ on a smooth 2-dimensional manifold $X$. The blob detection framework by \cite{ScalarBlob3D} is as follows:
\begin{enumerate} 
\item Calculate the scale-space $L(x,t):X\times \bbbr_{+} \toreal$. $L(x,t)$ is the solution of the heat equation on the surface
  $\partderiv{L(x, t)}{t}=-\LaplaceBeltrami{ L(x, t)},L(x, 0)=I(x)$, where $\LaplaceBeltrami$ is the Laplace-Beltrami operator;
\item Choose a blob response function and calculate it: 
\begin{equation} \textit{ the determinant blob response: } BR_{\det}(x, t)=\det{H_L(x,t)}\textrm{  or } \label{blob_det}\end{equation} 
\begin{equation} \textit{ the trace blob response: } BR_{\tr}(x, t)=\tr {H_L(x,t)}, \label{blob_tr}											\end{equation} 
 where $H_L$ is the Hessian of $L(x, t)$ as a function of $x$ with fixed $t$;
\item Find blobs centers and scales as $C=\{(x, t)=\arg \min_{x,t} \tilde{BR}(x, t)\textrm{  or } (x, t)=$
$=\arg \max_{x,t}\tilde{BR}(x, t)\}$, where 
$\tilde{BR}=t\,BR_{\tr}$ or $\tilde{BR}=t^2\,BR_{\det}$. Find the blobs radii as $s=\sqrt{2} t$.
\end{enumerate}

For the general case of a map between manifolds $I(x):X \to Y$ the Hessian is the covariant differential of the differential: $H_L=\CovariantDiff \Diff L,H_L \in \CotangentBundle{X}\otimes\CotangentBundle{X}\otimes\TangentBundle{Y}$. Consider the straightforward generalization of the blob detection stages:
\begin{enumerate}
\item Scale-space calculation. $L(x,t)$ is calculated as the solution of $\partderiv{L(x, t)}{t}=$ $=-\tr H_L(x, t), L(x, 0)=I(x)$. Methods of manifold-valued PDEs solution for different cases are discussed in the papers \cite{Harmonic,Kimmel,DTI} and others. These methods are out of scope of our work.
\item Blob response calculation. The determinant blob response $BR_{\det}=\det H_L$ is not defined. 
\item Blobs centers calculation. We can't find maximums or minimums of the trace blob response because it is not scalar-valued: $BR_{\tr}=\tr{H_L} \in TY$.
\end{enumerate} 

We see that there is no straightforward generalization of the blob response functions to the manifold-valued case. How can the problem of blob response generalization be solved? Our key ideas are the following:

\begin{enumerate}
\item	Consider the image graph $Gr$ as a submanifold embedded in $X\times Y$. The grayscale and manifold-valued cases differ only by a co-dimension of the embedding. Then a formulation of the blob response through notions defined for all co-dimensions will give an immediate generalization to the manifold-valued case. 
\item	What notions to use? The scalar and the mean curvatures are defined for all co-dimensions and are close to the determinant and the trace of the image Hessian respectively if tangent planes to $Gr$ and to $X$ are "close".
\end{enumerate}

\section{The Proposed Method}
\subsection{Used Notations}
All functions and manifolds here and further are considered to be smooth. Consider $m$- and $n$-dimensional manifolds $X$ and $Y$. Denote the $(n+m)$-dimensional manifold $X\times Y$ as $E$. Consider the isometric embeddings $i_x(y)=id(x, y):Y\to E, i_y(x)=id(x, y):X\to E$. Further we identify $X$ (resp. $Y$) and related notions with $i_y(X)$ (resp. $i_x(Y)$). The letters $i,j,k,l$ (resp. $\alpha, \beta, \gamma$) are used as indices for notions related to $X$ (resp. $Y$). The set $\{e_i\}$ (resp. $\{e_\alpha\}$) is an orthonormal basis of $T_x X$ (resp. $T_y Y$). 
\\
For a map $f(x):X\to Y$ its graph $Gr_f$ is an $n$-dimensional manifold embedded in $E$. Denote the Hessian of $f$ as $H_f$. Let $\mu Y$, $\mu \in \bbbr^{+}$, be the manifold $Y$ with the metric $\mu G_Y$. For a map $f:X\to Y$ denote $\mu f:X\to \mu Y$.
\\
We analyze a manifold-valued image $I(x):X \to Y$. Denote $L(x,t):X\times \bbbr_{+} \toreal$ the solution of the heat equation $\partderiv{L(x, t)}{t}=-\LaplaceBeltrami{ L(x, t)},L(x, 0)=I(x)$, where $\LaplaceBeltrami$ is the Laplace-Beltrami operator. 
\\
For a manifold $N$ and its submanifold $M$ denote the mean curvature of $M$ as $\tensor*{h}{^{N}_{M}}$, its scalar curvature as $r_M$, an exponential map from $\TangentSpaceArg{M}{m}$ to $N$ as $\tensor*{\exp}{^N_{M}}$.
\\
Subscripts and superscripts are omitted when they are clear from a context. The definitions of used differential geometric notions can be found in textbooks \cite{DiffGeom}.
\subsection{Main Definitions and Theorems}

\begin{definition} \label{RiemanDef}
The scalar blob response is defined as:
\begin{equation*}\FRScalar=\lim_{\mu\to 0} \frac{1}{\mu ^ 2} \Big( r_{Gr_{\mu L}} - r_{\tensor*{\exp}{^{X\times\mu Y}_{Gr_{\mu L}}}} \Big),\end{equation*}
the mean blob response is defined as:
\begin{equation*}\FRMean=\lim_{\mu\to 0} \frac{1}{\mu } \tensor*{h}{^{X\times\mu Y}_{Gr_{\mu L}}}.\end{equation*}
\end{definition}

The next theorem connects $\FRScalar$ and $\FRMean$ with the scale-space Hessian. The obtained expression provides a more convenient way for calculation of the Riemannian blob response functions.

\begin{theorem} \label{MainTheo}
Let $H_{ij}=H_L (e_i,e_j)$, $H^\alpha (,) =\langle H_L (,),e_\alpha {\rangle}_Y$. Then
\begin{equation*}\FRScalar=\sum_{i,j=1}^{n} \Big( \langle H_{ij},H_{ji}{\rangle}_Y-\langle H_{ii},H_{jj}{\rangle}_Y \Big) ,\end{equation*}
\begin{equation*}\FRMean=\| (\tr H^1,\dots,\tr H^m )\|_Y.\end{equation*}
\end{theorem}

The next corollary from Theorem \ref{MainTheo} states that for the grayscale case Riemannian blob detection coincides with usual blob detection. This corollary allows to consider our method as a generalization of grayscale blob detection.

\begin{corollary}\label{GrayscaleCol}
Let $dim(X)=2$. Then the scalar blob response is equal to the determinant blob response (\ref{blob_det}):
\begin{equation*}\FRScalar=BR_{\det},\end{equation*}
the mean blob response is equal to the trace blob response (\ref{blob_tr}):
\begin{equation*}\FRMean=BR_{\tr}.\end{equation*}
\end{corollary}

\subsection{Proof of the Theorem \ref{MainTheo}}
\subsubsection{Additional notations.}
Consider maps $y = f(x):X\to Y$, $\tilde{f}(x):X \to E$, $\tilde{f}(x)=(x,f(x))=\tilde{y}$. $\{e_i^{'}=\Diff \tilde{f}(e_i)\}$ is a basis (not orthonormal) of $T_{\tilde{y}} Gr_f$, 
$\{e_\alpha^{'}: (e_\alpha^{'},e_i^{'})_{E} = 0 \,\forall i \, \forall \alpha \}$ is a basis of $T_{\tilde{y}} (Gr_f)^{\bot}$. Then $\{e_i^{'}, e_\alpha^{'}\}$ is a basis of $T_{\tilde{y}} E$.
\\
For a manifold $M$ denote its metric as $g(∙,∙)_{M}$ or $\langle∙,∙\rangle_{M}$, the Levi-Civita connection as $\CovariantDiffManif{M}$, a connection on a vector bundle $\bigeps$ over $M$ as $\CovariantDiffManif{\bigeps}$. 
\\
Denote as $\Proj_V$ (resp. $\ProjNonOrth{U}{V}$) an orthogonal (resp. along a subspace $U$) projection on a subspace $V$.
\\
Some minor formal details of the proofs are omitted due to the space constraints.

\begin{proposition} \label{PropHessian}
Let $f:X\to Y$, $u, v\in T_x X$. Then
$H_f(u,v)=\CovariantDerivManif{\TPreimage{Y}{f}} {v} \Diff f(u) -$
\\
							$-\Diff f( 
							\CovariantDerivManif{X}{v}u
							)$.
							If $f$ is injective then 
							$H_f(u,v)=\CovariantDerivManif{Y}{ \Diff f(v)} \Diff f(u) - 
							\Diff f( 
							\CovariantDerivManif{X}{v}u
							)$.
\end{proposition}

\begin{proof}
Consider the Hessian $H_f$ as $H_f:\TangentBundle{X}\otimes\TangentBundle{X}\to \TPreimage{Y}{f}$, 
then $H_f(u, v)=$
\\
$= \sum_{\alpha=1}^m \CovariantDeriv{v}{(\Diff f(u, e_{\alpha}))} e_{\alpha}$. 
We apply the Leibniz rule to this expression and obtain the first statement. 
Recall that if $f$ is injective then $\Diff f$ is an isomorphism between $\TPreimage{Y}{f}$ and $TY$. This gives the second statement.
\qed
\end{proof}

\begin{lemma}  \label{LemSecondFormHessian}
Let $u, v \in T_xX$. Let $\CovariantDiffManif{\tilde{f}(X)}$ be the connection on $Gr_f$ induced by the isomorphism $\tilde{f}$. Let $II$ be the second fundamental form of the submanifold $Gr_f$ of $E$ with respect to the connection $\CovariantDiffManif{\tilde{f}(X)}$. Then 
$H_{\tilde{f}}(u, v) = II(\Diff \tilde{f}(u), \Diff \tilde{f}(v))$.
\end{lemma}
\begin{proof}
As $\tilde{f}$ is injective, by Prop. \ref{PropHessian}:
$H_{\tilde{f}}(u, v) = 
						\CovariantDerivManif{Y} {\Diff \tilde{f}(v)} \Diff \tilde{f}(u) - 
							\Diff \tilde{f}(\CovariantDerivManif{X}{v}u )$ = 
							\\
							$=\CovariantDerivManif{Y} {\Diff \tilde{f}(v)} \Diff \tilde{f}(u) - 
							\CovariantDerivManif{\tilde{f}(X)}{\Diff \tilde{f}(v)}\Diff \tilde{f}(u) = II(\Diff u, \Diff v)\,\textrm{\qed}$
\end{proof}

\begin{proposition} \label{PropBigSmallHess}
Let $u, v \in T_xX$, $\vec{0} \in T_xX$, $H_f(u, v) \in T_{f(x)}Y$.
Then $H_{\tilde{f}}(u, v) =$
\\
$= (\vec{0}, H_f(u, v))$.
\end{proposition}

\begin{proof}
By Proposition \ref{PropHessian}: 
$H_{\tilde{f}}(u, v)=\CovariantDerivManif{\Preimage{(X \times Y)}{\tilde{f}}} {v} \Diff \tilde{f}(u) - 
							\Diff \tilde{f}( 
							\CovariantDerivManif{X}{v}u
							)$.
\\
Recall that $\CovariantDerivManif{X\times Y}{(u_1, u_2)}{(v_1, v_2)} = (\CovariantDerivManif{X}{u_1}{v_1}, \CovariantDerivManif{Y}{u_2}{v_2})$. Then
							\\
							$\CovariantDerivManif{\Preimage{(X \times Y)}{\tilde{f}}}{v} (\Diff i_y, \Diff f)(u) 							
							 - (\Diff i_y( 
							\CovariantDerivManif{X}{v}u
							), \Diff f( 
							\CovariantDerivManif{X}{v}u
							)) = 
							\\
							=(\CovariantDerivManif{X}{v} (\Diff i_y u), \CovariantDerivManif{Y}{\Diff f(v)} (\Diff f u)) - 
							(\CovariantDerivManif{X}{v}u, \Diff f( 
							\CovariantDerivManif{X}{v}u)
							) =
							\\
							=(\vec{0}, \CovariantDerivManif{Y}{\Diff f(v)} (\Diff f u) - \Diff f( 
							\CovariantDerivManif{X}{v}u)) =
							(\vec{0}, H_f(u, v))$ \textrm{\qed }

\end{proof}

\begin{proposition} \label{PropProj}
Let $II_{\tilde{f}}$ be the second fundamental form of the submanifold $Gr_f$ of $E$ with respect to the connection $\CovariantDiffManif{\tilde{f}(X)}$ and $II_E$ be the second fundamental form with respect to the connection $\CovariantDiffManif{Gr_f}$ induced by $\CovariantDiffManif{E}$. 
Let $u, v \in T_{\tilde{y}}Gr_f$. Then $II_E(u, v) = \Proj_{T_{\tilde{y}}Gr_f^{\bot}} II_{\tilde{f}}(u, v)$.
\end{proposition}

\begin{proof}
By properties of a second fundamental form of a normalized manifold: 
$II_{E}(u, v) = \Proj_{T_{\tilde{y}}Gr_f^{\bot}}\CovariantDerivManif{E}{u}v$ and $\exists N \subset T_{\tilde{y}}E: 
II_{\tilde{f}}(u, v) = \ProjNonOrth{T_{\tilde{y}}Gr_f}{N}\CovariantDerivManif{E}{u}v$.
Then by simple operations with vectors we obtain the lemma proposition.
\end{proof}

\begin{lemma} \label{LemSecondHessianFormula}
$II_{Gr_f}(e_i^{'}, e_j^{'})=\sum_{\alpha, \beta=1}^{m} \tensor*{H}{^{\alpha}_i_j} \tensor*{g}{^{'}^{\alpha}^{\beta}} e_{\beta}^{'}$.
\end{lemma}

\begin{proof}
$\Proj_{T_{\tilde{y}}Gr_f^{\bot}}e_{\alpha} = \tensor*{g}{^{'}^{\alpha}^{\beta}} e_{\beta}^{'}$. Then 
$
II_{Gr_f}(e_i^{'}, e_j^{'}) = \textrm{(from Prop. \ref{PropProj} and Lemma \ref{LemSecondFormHessian})} 
\\
=\Proj_{T_{\tilde{y}}Gr_f^{\bot}} H_{\tilde{f}}(e_i, e_j)
=\textrm{(by Prop. \ref{PropBigSmallHess}) } \Proj_{T_{\tilde{y}}Gr_f^{\bot}} \tensor*{H}{^{\alpha}_f_i_j}e_{\alpha} 
= \sum_{\alpha, \beta=1}^{m} \tensor*{H}{^{\alpha}_i_j} \tensor*{g}{^{'}^{\alpha}^{\beta}} e_{\beta}^{'}
\textrm{\qed}
$
\end{proof}

\begin{proposition} \label{PropScaled}
Consider $\{ e_i^{'} \}$ as a basis of $T_{\tilde{y}}Gr_f$. 
$\Diff F$ is the matrix of $\Diff f$ in the basis $\{ e_i, e_{\alpha} \}$ and $E$ is the $n\times n$ unit matrix.
\\
Then:
the induced metric $g^{'}_{ij}$ on $Gr_f$ has the matrix $E + \Diff F^{T} \, \Diff F $;
the induced metric on $Gr_{\mu f}$ has the matrix $E + \mu^{2} \Diff F^{T} \, \Diff F$;
the covariant induced metric on $Gr_{\mu f}$ has the matrix $E - \mu^{2} \Diff F^{T} \Diff F + o(\mu ^ {2})$;
$H_{\mu f} = \mu H_f.$
\end{proposition}

\begin{lemma} \label{LemScalar}
$
\lim_{\mu\to 0} \frac{1}{{\mu}^2} \big( r_{Gr_{\mu f}} - r_{\tensor*{\exp}{^{X\times\mu Y}_ {Gr_{\mu f}} }} \big)
= \sum_{i,j=1}^{n} \Big( \langle H_{ij},H_{ji}{\rangle}_Y - \langle H_{ii},H_{jj}{\rangle}_Y \Big).
$
\end{lemma}

\begin{proof}
Write the Gauss equation, the scalar curvature definition and apply 
\\
Lemma\,\ref{LemSecondHessianFormula}: 
$r_{Gr_f} - r_{\tensor*{\exp}{^E_{Gr_{f}}}} = 
\tensor*{g}{^{'}^i^k}\tensor*{g}{^{'}^j^l}
\sum_{\alpha, \beta =1}^{m} g^{' \alpha \beta} 
\Big(\tensor*{H}{^{\alpha}_{ik}} \tensor*{H}{^{\alpha}_{jl}} - 
\tensor*{H}{^{\alpha}_{il}} \tensor*{H}{^{\alpha}_{jk}} \Big).$
Then substitute 
$\mu f$ for $f$ as $\mu$ tends to 0, apply Prop. \ref{PropScaled} and obtain the needed equality. \textrm{\qed}
\end{proof}

\begin{lemma} \label{LemMean}
$\lim_{\mu\to 0}\frac{1}{\mu} \tensor*{h}{^{X\times\mu Y}_{Gr_{\mu f}}} = \| (\tr H^1,\dots,\tr H^m )\|_Y$.
\end{lemma}

\begin{proof}
$ 
\tensor*{h}{^E_{Gr_f}} ^2 = \textrm{(by Lemma \ref{LemSecondHessianFormula}) }
\sum_{\gamma, \delta = 1}^{m} g^{'ii} g^{'ii} g^{'}_{\alpha \beta} g^{'\alpha\gamma} \tensor*{H}{_f^\gamma _i_i} g^{'\beta\delta} \tensor*{H}{_f^\delta _i_i}  .
$
For $\mu f$:
\\ 
$\lim_{\mu\to 0}\frac{1}{\mu} \tensor*{h}{^{X\times\mu Y}_{Gr_{\mu f}}} 
=\textrm{ (by Prop. \ref{PropScaled}) }
\lim_{\mu\to 0} \Big(\sum_{i, \alpha} \tensor*{H}{_f^\alpha _i_i} \tensor*{H}{_f^\alpha _i_i} + o(1) \Big)^{\frac{1}{2}} =
\\
= \| (\tr H^1,\dots,\tr H^m )\|_Y
\textrm{\qed}
$
\end{proof}

Theorem \ref{MainTheo} follows from Lemmas \ref{LemScalar} and \ref{LemMean}. The formulation of Theorem \ref{MainTheo} is obtained by substitution of $f$ with $L$.

\section{The Experiments}
\subsubsection{Experimental Setup.}
We apply our blob detection framework to a chemical compounds classification problem, called also the QSAR problem \cite{qsar}. The task is to distinguish active and non-active compounds using their structure. Each compound is represented by a triangulated molecular surface \cite{molecular} and several physico-chemical and geometrical properties on the surface. So an input data element can be modeled as a 2-dimensional manifold $X$ with a vector-valued function $f(x):X \to \bbbr^m$. We use the following properties: the electrostatic and the steric potentials, the Gaussian and the mean curvatures. These properties are calculated in each triangulation vertex.
\subsubsection{Implementation.}
We use Riemannian blob detection for the construction of descriptor vectors. The procedure is the following:
\begin{enumerate} 
\item Detect blobs by our method in each compound surface;
\item	Form pairs of blobs on each surface;
\item Transform the blobs pairs into vectors of fixed length by using the bag of words approach \cite{bag}.
\end{enumerate}
The Riemannian blob response functions are calculated for each triangulation vertex $v$. The procedure is the following:
\begin{enumerate} 
	\item Find the directional derivatives $\partderiv{L_i}{z_j}$ by the finite differences approximation, where $z_j$̅ are the directions from $v$ to its neighbour vertices. 
	\item Find the differential $\Diff L=(\Diff L_i)$ by solving the overdetermined linear system $\Diff L(Z)=\partderiv{L_i}{z_j}$ , $Z$  is a matrix which columns are vectors $z_j$.
	\item Find the covariant derivatives of the differential in the neighbour directions, i.e. find $\CovariantDerivManif{X}{z_j} \Diff L$ for each $j$ as by $\CovariantDerivManif{X}{z_j} \Diff L =\Proj_{T_xX} ( \CovariantDerivManif{\bbbr^3}{z_j} \Diff L$). $\CovariantDerivManif{\bbbr^3}{z_j} \Diff L$ are found by the finite differences approximation.
	\item Find the covariant differential $\CovariantDiffManif{X} \Diff L$ by solving the overdetermined linear system 
	$\CovariantDiffManif{X} \Diff L(Z)=\CovariantDerivManif{X}{z_j} \Diff L$ , $Z$  is a matrix which columns are vectors $z_j$.
	\end{enumerate}
	$\CovariantDiffManif{X} \Diff L=\{\tensor*{H}{_i_j^\alpha} \}$ is obtained. Calculate $\FRScalar(x,t)=\sum_{\alpha=1}^{m}\det H^{\alpha}$,
	\\
	$\FRMean(x,t)=\|\tr H^{\alpha}\|$.	
	\subsubsection{The Results.}
An example of the algorithm result is presented in Fig. \ref{fig:result}.
\\
We compare the prediction models built on the base of the following blob detection methods: 
\begin{enumerate}
\item	Riemannian blob detection with $\FRScalar$ as a blob response function;
\item	A naive method of applying blob detection to each channel separately;
\item	Riemannian blob detection with $\FRMean$ as a blob response function. It coincides with the method \cite{ColorBlob}, adapted to the case of 2D surface;
\item	The method of adaptive neighbourhood projection \cite{GROM}. It is adapted by us to the case of 2D surface.
\end{enumerate}
The feature reduction SVM \cite{SVM} is used for construction of the prediction model. The cross-validation functional \cite{cross} is used as an index of the performance quality. The test data is the following: 3 datasets (bzr, er\_lit, cox2) from \cite{kernel}, 3 datasets (glik, pirim, sesq) from Russian Oncology Science Center. The results are presented in Table \ref{tbl:result}.

\begin{figure}
\begin{floatrow}
\ffigbox{%
    \includegraphics[scale=0.25,width=170pt,height=93pt]{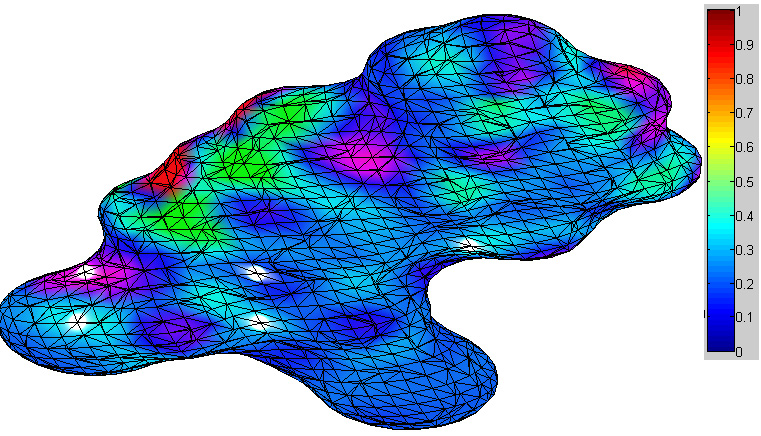}  
}{%
  \caption{A molecular surface with $\FRScalar$ on it and found centers (denoted by white color) 
	of blobs of radii 3.}
  \label{fig:result}
}
\capbtabbox{%
  \begin{tabular}{| l | l | l | l | l |}
    \hline
		 & $\FRScalar$ & naive & $\FRMean$ &	Adapt. \cite{GROM} \\ \hline
		glik	& 1.0 &	0.954 &	0.975 &	1.0 \\ \hline
		pirim	& 0.99 & 0.96 &	0.97 &	0.98 \\ \hline
		sesq	& 1.0 &	0.98 &	0.976 &	1.0 \\ \hline
		bzr	& 0.992 &	0.971 &	0.975 &	0.983 \\ \hline
		er\_lit & 0.98 &	0.961 &	0.956 &	0.98 \\ \hline
		cox2 & 0.991 & 0.967 & 0.985 &	0.986 \\ \hline    
  \end{tabular}
}{%
  \caption{The results: the cross-validation of the models, based on feature vectors built by the blob detection methods.}%
	\label{tbl:result}
}
\end{floatrow}
\end{figure}

Riemannian blob detection with $\FRScalar$ as a blob response function is the best performing method. This shows the effectiveness of our approach. This particular method for vector-valued functions on 2D surfaces wasn't presented in the literature before. 

\section{Conclusion and Future work}
We propose the Riemannian framework for blob detection in manifold-valued images. This framework is based on the definition of the blob response functions by means of the image graph curvatures. Our approach gives new methods for the uncovered problems and coincides with classical blob detection for the grayscale case. The experiments results show the effectiveness of the proposed approach.
\\
The next direction for the research is a generalization of our framework to the case of sections of non-trivial fiber bundles. In particular, such generalization will cover an important case of tangent vector fields.

\section*{Acknowledgments}
The authors want to thank Dr. Alexey Malistov for valuable discussions and the help with the article editing.

\end{document}